%% file: gpnn_arxiv.tex
\documentclass[letterpaper]{article}
\usepackage{proceed2e}
\usepackage[margin=1in]{geometry}

\usepackage[hyphens]{url}
\usepackage{hyperref}
\usepackage{graphicx}
\usepackage{amsmath}
\usepackage{amsthm}
\usepackage{amssymb}
\usepackage{algorithm}
\usepackage{algpseudocode}
\usepackage[labelformat=simple]{subcaption}
\usepackage{booktabs}
\usepackage[square,sort,comma,numbers]{natbib}

\usepackage{times}

\title{Graph Partition Neural Networks for Semi-Supervised Classification}

\author{
{Renjie Liao\thanks{~Work done partially while author was at Microsoft Research.}} \\
University of Toronto \\
Uber ATG Toronto \\
Vector Institute \\
rjliao@cs.toronto.edu \\
\And
{Marc Brockschmidt} \\
Microsoft Research Cambridge \\
mabrocks@microsoft.com \\
\And
{Daniel Tarlow$^{\ast}$} \\
Google Brain \\
dtarlow@google.com \\
\AND
{Alexander L. Gaunt} \\
Microsoft Research Cambridge \\
algaunt@microsoft.com \\
\And
{Raquel Urtasun} \\
University of Toronto \\
Uber ATG Toronto \\
Vector Institute \\
urtasun@cs.toronto.edu \\
\And
{Richard Zemel} \\
University of Toronto\\ 
Vector Institute \\
Canadian Institute for Advanced Research \\
zemel@cs.toronto.edu
} 

\newcommand{\comment}[1]{}

\newcommand{\ourmodelshort}{GPNN}

\newtheorem{prop}{Proposition}

\algblockdefx{FORALLP}{ENDFAP}[1]{\textbf{for all }#1 \textbf{do in parallel}}{}
\algtext*{EndWhile}
\algtext*{EndIf}
\algtext*{EndFor}
\algtext*{EndFunction}
\algtext*{ENDFAP}

\newcommand{\graph}{\mathcal{G}}
\newcommand{\subgraph}{\mathcal{S}}
\newcommand{\nodes}{\mathcal{V}}

\newcommand{\edges}{\mathcal{E}}

\newcommand{\statesymbol}{\boldsymbol{h}}
\newcommand{\state}[2]{\ensuremath{\statesymbol^{(#1)}_{#2}}}

\newcommand{\msg}{\boldsymbol{m}}
\newcommand{\outneighbors}[0]{\ensuremath{\mathcal{N}_{\mathit{out}}}}
\newcommand{\inneighbors}[0]{\ensuremath{\mathcal{N}_{\mathit{in}}}}

\newcommand{\rSC}[1]{Sect.~\ref{#1}}
\newcommand{\rF}[1]{Fig.~\ref{#1}}

\newcommand{\rA}[1]{Alg.~\ref{#1}}

\newcommand{\rTab}[1]{Tab.~\ref{#1}}

\begin{document}

\maketitle

\begin{abstract}
We present graph partition neural networks (GPNN), an extension of graph neural
networks (GNNs) able to handle extremely large graphs.
GPNNs alternate between locally propagating information between nodes in small
subgraphs and globally propagating information between the subgraphs.
To efficiently partition graphs, we experiment with several partitioning algorithms and also
propose a novel variant for fast processing of large scale graphs.
We extensively test our model on a variety of semi-supervised node
classification tasks.
Experimental results indicate that GPNNs are either superior or comparable to 
state-of-the-art methods on a wide variety of datasets for graph-based 
semi-supervised classification. We also show that GPNNs can achieve similar performance as standard GNNs with
fewer propagation steps.
\end{abstract}

\input{introduction}

\input{related_work}

\input{model}

\input{experiments}

\input{discussion}



\bibliography{gpnn}
\bibliographystyle{ieee}

\clearpage
\appendix
\input{appendix}

\end{document}

%% file: introduction.tex
\section{Introduction}


Graphs are a flexible way of encoding data, and many tasks can be
cast as learning from graph-structured inputs.
Examples include
prediction of properties of chemical molecules~\citep{duvenaud2015convolutional},
answering questions about knowledge graphs~\citep{Marino16},
natural language processing with parse-structured inputs (trees or richer structures like Abstract Meaning Representations)~\citep{banarescu2012abstract},
predicting properties of data structures or source code in programming languages~\citep{li2015gated,allamanis18learning}, and
making predictions from scene graphs~\citep{Teney16}.
Sequence data can be seen as a special case of a simple chain-structured graph.
Thus, we are interested in training high-capacity neural network-like
models on these types of graph-structured inputs.
Graph Neural Networks (GNNs)~\citep{gori2005new,scarselli2009graph,li2015gated,qi20173d,li2017situation,garcia2018few,wang18nervenet}
are one of the best contenders, although there has been much recent interest in
applying other neural network-like models to graph data, including
generalizations of convolutional
architectures~\citep{duvenaud2015convolutional,kipf2016semi,schlichtkrull2017modelling}.
Gilmer et al.~\cite{Gilmer17} recently reviewed and unified many of these models.

An important issue that has not received much attention in GNN models is how
information gets propagated across the graph.
There are often scenarios in which information has to be propagated over long
distances across a graph, e.g.,\ when we have long sequences augmented with
additional relationships between elements of the sequence, like in text,
programming language source code, or temporal streams.
The simplest approach, and the one adopted by almost all graph-based neural
networks is to follow \emph{synchronous message-passing
  systems}~\citep{attiya2004distributed} from distributed computing theory.
Specifically, inference is executed as a sequence of rounds: in each round,
every node sends messages to all of its neighbors, the messages are delivered
and every node does some computation based on the received messages.
While this approach has the benefit of being simple and easy to implement, it is
especially inefficient when the task requires spreading information across long
distances in the graph.
For example, in processing sequence data, if we were to employ the above
schedule for a sequence of length $N$, it would take $O(N^2)$ messages
to propagate information from the beginning of the sequence to the end, and during training all $O(N^2)$ messages must be stored in memory.
In contrast, the common practice with sequence data is to use a forward pass
followed by a backward pass at a cost of $O(N)$ to propagate information from
end to end, as in hidden Markov models (HMMs) or bidirectional recurrent neural networks (RNNs), for example.

One possible approach for tackling this problem is to propagate information over
the graph following some pre-specified sequential order, as in
Bidirectional LSTMs.
However, this sequential solution has several issues.
First, if a graph used for training has large diameter, the unrolled GNN
computational graph will be large (cf. Bidirectional LSTMs on long sequences). This leads
to fundamental issues with learning (e.g., vanishing/exploding gradients) and implementation difficulties (i.e., resource constraints).
Second, sequential schedules are typically less amenable to efficient acceleration on parallel hardware.
More recently, Gilmer et al.~\cite{Gilmer17} attempted to tackle the first problem by introducing a
``dummy node'' with connections to all nodes in the input graph, meaning that
all nodes are at most two steps away from each other.
However, we note that the graph structure itself often contains important
information, which is modified by adding additional nodes and edges.

In this work, we propose graph partition neural networks (GPNN) that exploit a propagation schedule combining features of synchronous and
sequential propagation schedules.
Concretely, we first partition the graph into disjunct subgraphs and a cut
set, and then alternate steps of synchronous propagation within subgraphs with
synchronous propagation within the cut set.
In \rSC{sec:model}, we discuss different propagation schedules on an example,
showing that GPNNs can be substantially more efficient than standard GNNs,
and then present our model formally.
Finally, we evaluate our model in \rSC{sec:experiments} on a variety
of semi-supervised classification benchmarks.
The empirical results suggest that our models are either superior to or
comparable with state-of-the-art learning systems on graphs.


%% file: related_work.tex
\section{Related Work}


There are many neural network models for handling graph-structured inputs. They
can be roughly categorized into generalizations of recurrent neural networks
(RNNs)
\citep{goller1996learning,gori2005new,scarselli2009graph,socher2011parsing,tai2015improved,li2015gated,Marino16,
  qi20173d,li2017situation} and generalizations of convolutional neural networks
(CNNs)
\citep{bruna2013spectral,duvenaud2015convolutional,kipf2016semi,schlichtkrull2017modelling}.
Gilmer et al.~\cite{Gilmer17} provide a good review and unification of many of these models, and they present some additional model variations that lead to strong empirical results in making predictions from molecule-structured inputs.

In RNN-like models, the standard approach is to propagate information using a
synchronous schedule. In convolution-like models, the node updates mimic
standard convolutions where all nodes in a layer are updated as functions of
neighboring node states in the previous layer. This leads to information
propagating across the graph in the same pattern as synchronous schedules. While
our focus has been mainly on the RNN-like model of Li et al.~\cite{li2015gated},
it would be interesting to apply our schedules to the other models as well.

Some of the RNN based neural network models operate on restricted classes of graphs and employ sequential or sequential-like schedules. For example, recursive neural networks \citep{goller1996learning,socher2011dynamic} and tree-LSTMs \cite{tai2015improved} have bidirectional variants that use fully sequential schedules.
Sukhbaatar et al.~\cite{sukhbaatar2016learning} modeling of agents can be viewed
as a GNN model with a sequential schedule, where messages are passed inwards
towards a master node that aggregates messages from different agents, and then
outwards from the master node to all the agents. The difference in our work is
the focus on graphs with arbitrary structure (not necessarily a sequence or
tree). Recently, Marino et al.~\cite{Marino16} developed an attention-like
mechanism to dynamically select a subset of graph nodes to propagate information
from, but the propagation is synchronous amongst selected nodes.

Recently, Hamilton et al.~\cite{hamilton2017inductive} propose a graph sample and aggregate (GraphSAGE) method.
It first samples a neighborhood graph for each node which can be regarded as overlapping partitions of the original graph.
An improved graph convolutional network (GCN)~\cite{kipf2016semi} is then applied to each neighborhood graph independently.
They show that this partition based strategy facilitates the unsupervised representation learning on large scale graphs.


An area where scheduling has been studied extensively is in the probabilistic inference literature. 
It is common to decompose a graph into a set of spanning trees and sequentially update the tree structures \cite{wainwright2002tree}. 
Graph partition based schedules have been explored in belief propagation (BP) \cite{pearl1988probabilistic}, generalized belief propagation (GBP) \cite{yedidia2003understanding,welling2004choice}, generalized mean-field inference \cite{xing2002generalized,xing2004graph} and dual decomposition based inference~\cite{komodakis2011mrf,zhang2014message}.
In generalized mean-field inference \cite{xing2002generalized}, a graph partition algorithm, e.g., graph cut, is applied to obtain the clusters of nodes. 
A sequential update schedule among clusters is adopted to perform variational inference.
Zhang et al.~\cite{zhang2014message} adopt a partition-based strategy to better distribute the dual decomposition based message passing algorithm for high order MRF. 
The junction tree algorithm~\cite{lauritzen1988local} can also be viewed as a partition based inference where the partition is obtained by finding the maximum spanning tree on the weighted clique graph.
Each node of the junction tree corresponds to a cluster of nodes, i.e., maximal clique, in the original graph. 
A sequential update can then be executed on the junction tree.
See also \cite{elidan2006residual,Tarlow11graphproduct,sutton2012improved} for more discussion of sequential updates in the context of belief propagation. Finally, the question of sequential versus synchronous updates arises in numerical linear algebra. Jacobi iteration uses a synchronous update while Gauss-Seidel applies the same algorithm but according to a sequential schedule.


%% file: model.tex
\section{Model}
\label{sec:model}

In this section, we briefly recapitulate graph neural networks (GNNs) and
then describe our graph partition neural networks (GPNN).
A graph $\graph = (\nodes, \edges)$ has nodes $\nodes$ and edges $\edges
\subseteq \nodes \times \nodes$.
We focus on directed graphs, as our approach readily
applies to undirected graphs 
by splitting any undirected edge into two directed edges.
We denote the out-going neighborhood as
 $\outneighbors(v) = \{ u \in \nodes \mid (v, u) \in \edges \}$,
and similarly, the incoming neighborhood as
 $\inneighbors(v) = \{ u \in \nodes \mid (u, v) \in \edges \}$.
We associate an edge type $c_{(v,u)} \in \{1, \ldots, C\}$ with every edge $(v,
u)$, where $C$ is some pre-specified total number of edge types. Such edge
types are used to encode different relationships between nodes. Note that one
can also associate multiple edge types with the same edge which results in a
multi-graph. 
We assume one edge type per directed edge to simplify the notation here, but the
model can be easily generalized to the multi-edge case.

\subsection{Graph Neural Networks}

Graph neural networks~\citep{scarselli2009graph,li2015gated} can be
viewed as an extension of recurrent neural networks (RNNs) to arbitrary graphs.
Each node $v$ in the graph is associated with an initial state vector
$\state{0}{v}$ at time step $0$.
Initial state vectors can be observed features or annotations as in
\cite{li2015gated}.
At time step $t$, an outgoing message is computed for each edge by transforming
the source state according to the edge type, i.e., 
\begin{align}\label{model:msg_func}
\msg_{(v,u)}^{(t)} = M_{c_{(v, u)}}(\state{t}{v}),
\end{align}
where $M_{c_{(u, v)}}$ is a message function, which could be the identity or a
fully connected neural network.
Note the subscript $c_{(v, u)}$ indicating that different edges of the same type
share the same instance of the message function.
We then aggregate all messages at the receiving nodes, i.e.,
\begin{align}\label{model:msg_agg}
\bar{\msg}_{u}^{(t)} = A(\{\msg_{(v, u)}^{(t)} \mid v \in \inneighbors(u)\}),
\end{align}
where $A$ is the aggregation function, which may be a summation, average or
max-pooling function.
Finally, every node will update its state vector based on its current state
vector and the aggregated message, i.e.,
\begin{align}\label{model:update}
\state{t+1}{v} = U(\state{t}{v}, \bar{\msg}_{v}^{(t)}),
\end{align}
where $U$ is the update function, which may be a gated recurrent unit (GRU), a
long short term memory (LSTM) unit, or a fully connected network.
Note that all nodes share the same instance of update function.
The described propagation step is repeatedly applied for a fixed number of time
steps $T$, to obtain final state vectors $\{\state{T}{v} \mid v \in
\nodes \}$.
A node classification task can then be implemented by feeding these state
vectors to a fully connected neural network which is shared by all nodes.
Back-propagation through time (BPTT) is typically adopted for learning the model.

\subsection{Graph Partition Neural Networks}

The above inference process is described from the perspective of an individual
node.
If we look at the same process from the graph view, we observe a
\emph{synchronous schedule} in which all nodes receive and send messages at the
same time, cf.\ the illustration in \rF{fig:graph}(d).
A natural question is to consider different propagation schedules in which not
all nodes in the graph send messages at the same time, e.g., \emph{sequential
  schedules}, in which nodes are ordered in some linear sequence and
messages are sent only from one node at a time.
A mix of the two ideas leads to our Graph Partition Neural Networks (GPNN),
which we will discuss before elaborating on how to partition graphs
appropriately. Finally, we discuss how to handle initial node labels and node
classification tasks.

\input{model_figure}

\paragraph{Propagation Model}

\begin{algorithm}[t]
\caption{Graph Partition Propagation Schedule.}
\label{alg:propagation}
\begin{algorithmic}[1]
    \State \textbf{Input}:
      $K$ subgraphs $\{\subgraph_k \vert k = 1, \dots, K\}$,
      cut $\subgraph_0$,
      outer propagation step limit $T$,
      intra-subgraph and inter-subgraph propagation step limits $T_S$ and $T_C$.
    
    \For{$t = 1, \dots, T$}
        \FORALLP{$k \in \{1, \dots, K\}$}            
            \State Execute $\textproc{SyncProp}(\subgraph_k)$ $T_S$ times.
        \ENDFAP
        \State Execute $\textproc{SyncProp}(\subgraph_0)$ $T_C$ times.
    \EndFor
    \Function{SyncProp}{Graph $\graph$}
      \State Compute \& send messages as in Eq. (\ref{model:msg_func})
      \State Aggregate messages as in Eq. (\ref{model:msg_agg})
      \State Update states as in Eq. (\ref{model:update})
    \EndFunction
\end{algorithmic}
\end{algorithm}

We first consider the example graph in \rF{fig:graph} (a).
A corresponding computational graph that shows how information is propagated
from time step $t$ to time step $t+1$ using the standard (synchronous) propagation
schedule is shown in \rF{fig:graph} (d).
The example graph's diameter is $5$, and it hence requires at least $5$ steps to
propagate information over the graph.
\rF{fig:graph}(c) instead shows two possible sequences that show how information
can be propagated between nodes $2$ to $6$ and $5$ to $1$.
These visualizations show that
 (i) a full synchronous propagation schedule requires significant computation at
     each step, and
 (ii) a sequential propagation schedule, in which we only propagate along
      sequences of nodes, results in very sparse and deep computational graphs.
Moreover, experimentally, we found sequential schedules to require multiple
propagation rounds across the whole graph, resulting in an even deeper
computational graph.

In order to achieve both efficient propagation and tractable learning, we
propose a new propagation schedule that follows a divide and conquer strategy.
In particular, we first partition the graph into disjunct subgraphs.
We will explain the details of how to compute graph partitions below.
For now, we assume that we already have $K$ subgraphs such that each
subgraph contains a subset of nodes and the edges induced by this subset.
We will also have a cut set, i.e., the set of edges that connect different
subgraphs.
One possible partition of our example is visualized in \rF{fig:graph} (b).

In GPNNs, we alternate between propagating information in parallel local to each
subgraph (making use of highly parallel computing units such as GPUs) and
propagating messages between subgraphs.
Our propagation schedule is shown in Alg. \ref{alg:propagation}.
To understand the benefit of this schedule, consider a
broadcasting problem over the example graph in Fig. \ref{fig:graph}. When information from any one node has reached all other nodes in the graph for the first time, this problem is considered as solved. We will compare the number of messages required to solve this problem for different propagation schedules.

\textit{Synchronous propagation}: Fig. \ref{fig:graph}(d) shows that a synchronous step requires 10 messages. Broadcasting requires sufficient propagation steps to cover the graph diameter (in this case, 5), giving a total of $5\times 10 = 50$ messages.

\textit{Partitioned propagation}: For simplicity, we analyze the case $T_S = D_S$, $T_C = 1$, where $D_S$ is the maximum diameter of the subgraphs. Using the partitioning in \ref{fig:graph}(e), we have $D_S=2$ and each step of intra-subgraph propagation requires 8 messages. After $T_S$ steps ($8D_S$ messages) the broadcast problem is solved within each subgraph. Inter-subgraph propagation requires 2 messages in this example, giving $8D_S+2$ messages per outer loop iteration in Alg. \ref{alg:propagation}. The example requires $2$ outer iterations to broadcast between all subgraphs, giving a total of $2(8D_S+2)=36$ messages.


In general, our propagation schedule requires no more messages than the
synchronous schedule to solve broadcast (if the number of subgraphs $K$ is 
set to $1$ or $N$ then our schedule reduces to the synchronous one).
We analyze the number of messages required to solve the broadcast problem on
chain graphs in detail in \rSC{sec:chain}.
Overall, our method avoids the large number of messages required by synchronous
schedules, while avoiding the very deep computational graphs required by
sequential schedules.
Our experiments in \rSC{sec:experiments} show that this makes learning tractable
even on extremely large graphs.

\paragraph{Graph Partition}

We now investigate how to construct graph partitions.
First, since partition problems in graph theory typically are NP-hard, we
are only looking for approximations in practice.
A simple approach is to re-use the classical spectral partition method.
Specifically, we follow the normalized cut method in \cite{shi2000normalized}
and use the random walk normalized graph Laplacian matrix $L = I - D^{-1} W$, where $I$ is
the identity matrix, $D$ is the degree matrix and $W$ is the weight matrix of
graph (i.e., the adjacency matrix if no weights are presented).

However, the spectral partition method is slow and hard to scale with large
graphs~\citep{von2007tutorial}.
For performance reasons, we developed the following heuristic method based on a
multi-seed flood fill partition algorithm as listed in \rA{alg:partition}.
We first randomly sample the initial seed nodes biased towards
nodes which are labeled and have a large out-degree.
We maintain a global dictionary assigning nodes to subgraphs, and initially
assign each selected seed node to its own subgraph.
We then grow the dictionary using flood fill, attaching unassigned nodes
that are direct neighbors of a subgraph to that graph.
To avoid bias towards the first subgraph, we randomly permute the order of
subgraphs at the beginning of each round.
This procedure is repeatedly applied until no subgraph grows anymore.
There may still be disconnected components left in the graph, which we assign to
the smallest subgraph found so far to balance subgraph sizes.

\begin{algorithm}[t]
\caption{Modified Multi-seed Flood Fill Partition Algorithm.}
\label{alg:partition}
\begin{algorithmic}[1]
    \State \textbf{Input}: Graph $G$, number of subgraphs $K$, indices $I$ of nodes which are labeled.
    \State Create two dictionaries $D$ and $L$ and $K$ FIFO queues $Q = \{Q_1, \dots, Q_K\}$. $D$ maps node index to \textproc{False} and $L$ maps node index to subgraph index $0$.
    \State $\forall u \in I$, compute the out-going degree $d_u$ of node $u$.
    \State $\forall u \in I$, compute the probability $p_u = d_u / \sum_{v \in I} d_v$.
    \State Sample $K$ nodes $S=\{s_1, \dots, s_K\}$ from $I$ based on the above probability distribution $p$.
    \State $\forall s_k \in S$, enqueue $s_k$ to $Q_k$, $D(s_k) = \textproc{True}$, $L(s_k) = k$.

    \While{Any queue in $Q$ is not empty}
        \For{$k \in \textproc{RandPerm(K)}$}
            \If{$Q_k$ is not empty}
                \State $u \leftarrow $ pop $Q_k$
                \For{$v \in \textproc{Children}(u)$}
                    \If{$D(v) == \textproc{False}$}
                        \State Enqueue $v$ to $Q_k$
                        \State $L(v) = k$
                        \State $D(v) = \textproc{True}$
                    \EndIf
                \EndFor
            \EndIf
        \EndFor
    \EndWhile

    \State Put any unvisited nodes into the smallest subgraph and set $L$ accordingly.
    \State Return $L$
\end{algorithmic}
\end{algorithm}

\paragraph{Node Features \& Classification}

In practice, problems using graph-structured data sometimes
 (1) do not have observed features associated with every
     node~\citep{grover2016node2vec};
 (2) have very high dimensional sparse features per node~\citep{bing2015improving}.
We develop two types of models for the initial node labels:
\emph{embedding-input} and \emph{feature-input}.
For \emph{embedding-input}, we introduce learnable node embeddings into the
model to solve challenge (1), inspired by other graph embedding methods. For nodes with
observed features we initialize the embeddings to these observations, and all other nodes are initialized randomly.
All embeddings are fed to the propagation model and are treated as learnable parameters.
For \emph{feature-input}, we apply a sparse fully-connected network to input
features to tackle challenge (2).
The dimension-reduced feature is then fed to the propagation model, and the
sparse network is jointly learned with the rest of model.

We also empirically found that concatenating the input features with the final
embedding produced by the propagation model is helpful in boosting the
performance.


%% file: model_figure.tex

\begin{figure*}[t]
\vspace{-1.0cm}
\begin{minipage}{0.445\textwidth}
\begin{subfigure}{0.49\textwidth}
\centering
\vspace{-0.25cm}
\includegraphics[width=\linewidth]{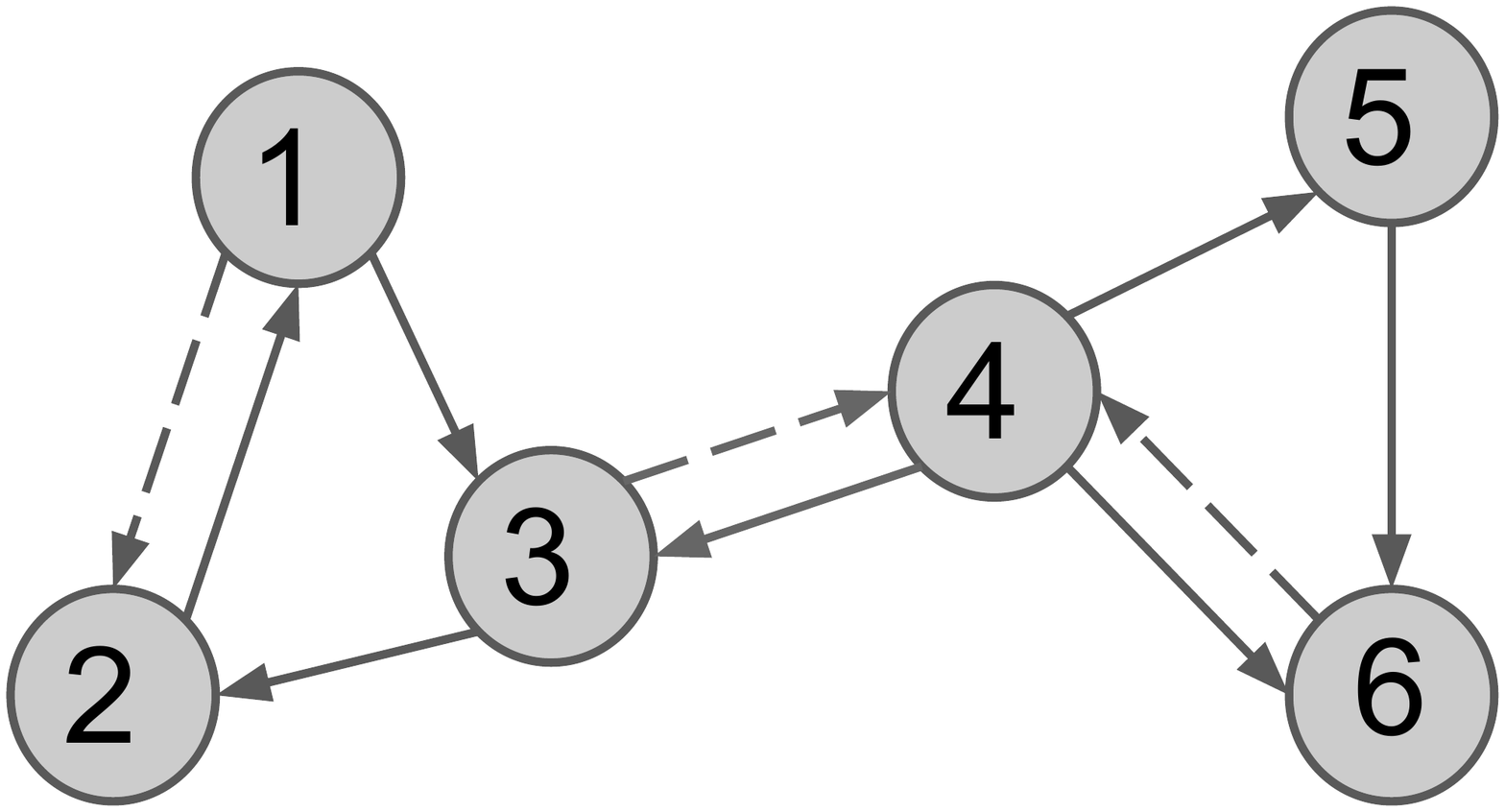} \\
\vspace{-1.0cm}
(a) 
\end{subfigure}
\hfill
\begin{subfigure}{0.49\textwidth}
\centering
\vspace{-0.25cm}
\includegraphics[width=\linewidth]{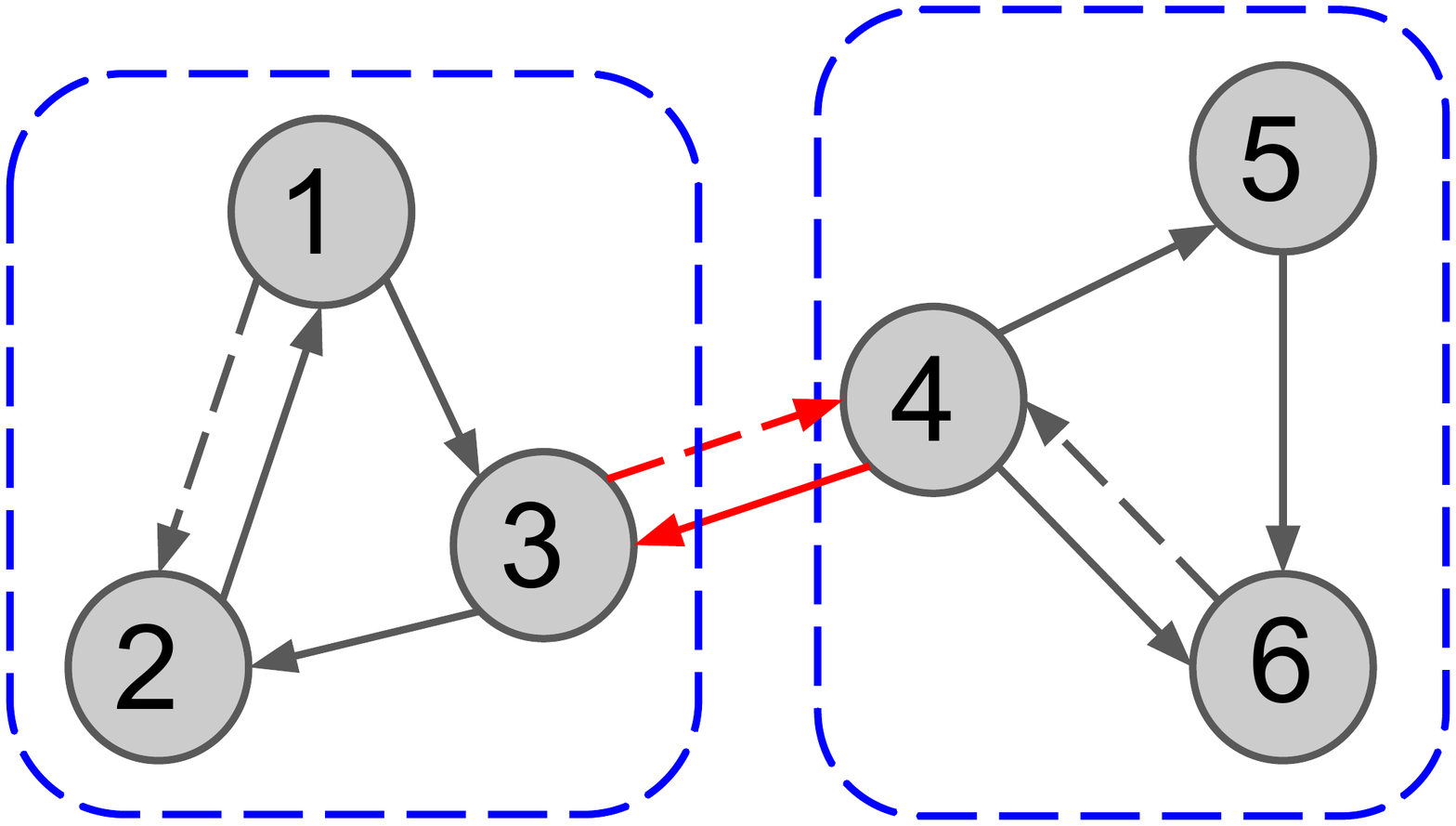} \\
\vspace{-1.0cm}
(b)
\end{subfigure}

\bigskip 
\begin{subfigure}{0.9\textwidth}
\centering
\vspace{-0.25cm}
\includegraphics[width=\linewidth]{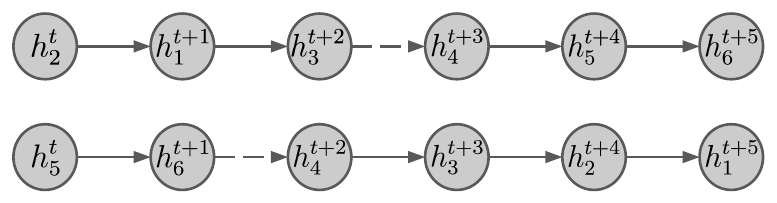} \\
(c)
\end{subfigure}
\end{minipage}
\hfill
\begin{minipage}{0.25\textwidth}
\hspace{-0.75cm}
\begin{subfigure}{\textwidth}
\centering
\includegraphics[height=1.15\linewidth]{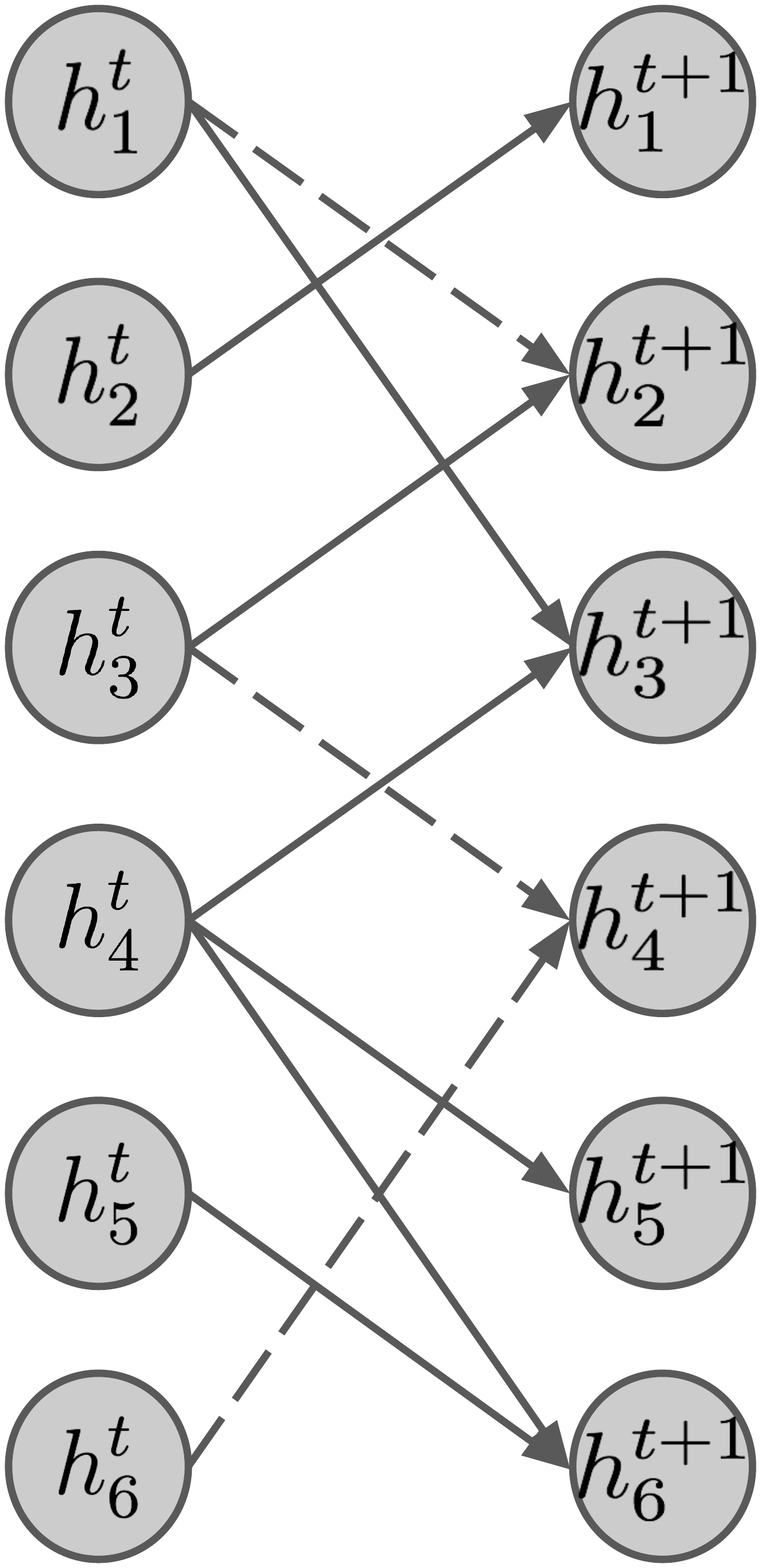} \\
(d)
\end{subfigure}
\end{minipage}
\hfill
\begin{minipage}{0.25\textwidth}
\hspace{-1.5cm}
\begin{subfigure}{\textwidth}
\centering
\includegraphics[height=1.15\linewidth]{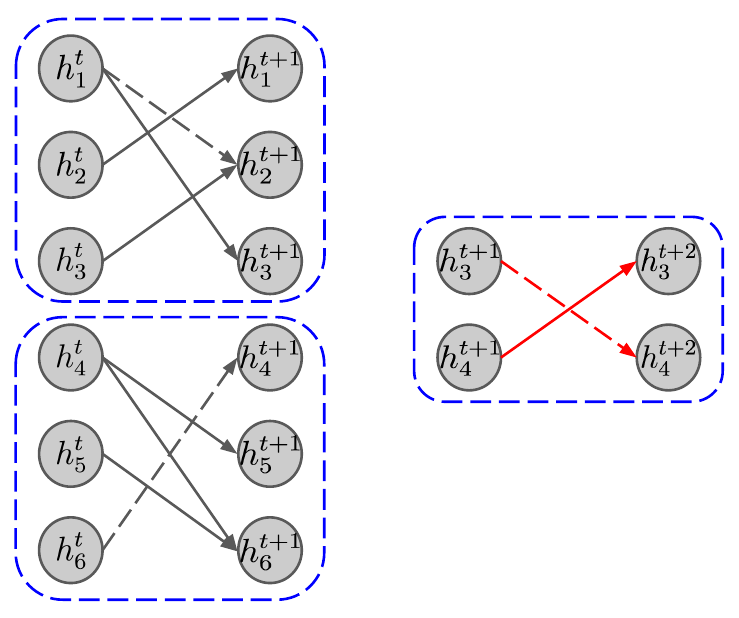} \\
(e)
\end{subfigure}
\end{minipage}  
\caption{Propagation schedules on an example graph. (a) The input graph where the line type, i.e., solid \& dash, indicates different edge types; (b) Graph partitions where blue bounding boxes indicate different subgraphs and red edges belong to the cut; (c) Computational graphs of two possible sequential propagation schedules of the input graph; (d) Computational graph for synchronous propagation schedule; (e) Computational graph for GPNNs where both inter-subgraph and intra-subgraph propagation steps are $1$.}
\label{fig:graph}
\end{figure*}

%% file: experiments.tex
\section{Experiments}
\label{sec:experiments}

We test our model on a variety of semi-supervised tasks
\footnote{Our code is released at \url{https://github.com/Microsoft/graph-partition-neural-network-samples}}
: document classification
on citation networks; entity classification in a bipartite graph extracted from
a knowledge graph; and distantly-supervised entity extraction.
We then compare different partition methods exploited by our model. 
We also compare the effectiveness of different propagation schedules. 
We follow the datasets and experimental setups in \cite{yang2016revisiting}. 
The statistics are summarized in \rTab{exp:datasets}, revealing that the datasets vary a lot in terms of scale, label rate and feature dimension. 
We report the details of hyper-parameters for all experiments in the appendix.


\begin{table}
\centering
\resizebox{\columnwidth}{!}{%
\begin{tabular}{@{}l|rrrrrr@{}}
\hline
\toprule
Dataset & \#Nodes & \#Edges & \#Classes & \#Features & Label Rate  \\
\midrule
Citeseer & 3,327 & 4,732 & 6 & 3,703 & 0.036\phantom{$^\ast$} \\
Cora & 2,708 & 5,429 & 7 & 1,433 & 0.052\phantom{$^\ast$} \\
Pubmed & 19,717 & 44,338 & 3 & 500 & 0.003\phantom{$^\ast$} \\
NELL & 65,755 & 266,144 & 210 & 5,414 & 0.1, 0.01, 0.001\phantom{$^\ast$} \\
DIEL & 4,373,008 & 4,464,261 & 4 & 1,233,598 & 0.0095$^\ast$ \\
\bottomrule
\end{tabular}
}
\caption{Dataset statistics. $^\ast$ indicates the average label rate over $10$ fixed splits.}
\label{exp:datasets}
\end{table}

\subsection{Citation Networks}

We first discuss experimental results on three citation networks: Citeseer, Cora and Pubmed \citep{sen2008collective}. The datasets contain sparse bag-of-words feature vectors for each document and a list of citation links between documents. Documents and citation links are regarded as nodes and edges while constructing the graph. $20$ instances are sampled for each class as labeled data, 1000 instances as test data, and the rest are used as unlabeled data. The goal is to classify each document into one of the predefined classes. We use the same data split as in \cite{yang2016revisiting} and \cite{kipf2016semi}. We use an additional validation set of 500 labeled nodes for tuning hyperparameters as in \cite{kipf2016semi}.

The experimental results are shown in \rTab{exp:citation}.
We report the results of baselines directly from \cite{yang2016revisiting} and
\cite{kipf2016semi}.
We see that GPNN is on par with other state-of-the-art methods on these small graphs.
We also conducted experiments with $10$ random splits and results are reported in the appendix.
We found these datasets easy to overfit due to their small size, and use
\emph{feature-input} rather than \emph{embedding-input}, as the latter case
increases the model capacity as well as the risk of overfitting.
We also show a t-SNE \citep{maaten2008visualizing} visualization of node
representations produced by the propagation model of GGNN and \ourmodelshort~on
the Cora dataset in Fig. \ref{fig:visualization} (a) and (b) respectively.
The visualizations show that the node representations of \ourmodelshort~are better separated.

\begin{table}
\centering
\resizebox{\columnwidth}{!}{%
\begin{tabular}{@{}l@{}c@{}c@{}c@{}c@{}c@{}c@{}c@{}lll@{}}
\hline
\toprule
Method &~~~~& Citeseer &~~~~& Cora &~~~~& Pubmed &~~~~~~~& \multicolumn{3}{c}{NELL} \\
      &&          &&      &&        &&  $10\%$ & $1\%$ & $0.1\%$\\
\cmidrule{1-1} \cmidrule{3-7} \cmidrule{9-11}
Feat\citep{yang2016revisiting}                     && 57.2 && 57.4 && 69.8 && 62.1 & 40.4 & 21.7 \\
ManiReg\citep{belkin2006manifold}                  && 60.1 && 59.5 && 70.7 && 63.4 & 41.3 & 21.8\\
SemiEmb\citep{weston2012deep}                      && 59.6 && 59.0 && 71.1 && 65.4 & 43.8 & 26.7\\
LP\citep{zhu2003semi}                              && 45.3 && 68.0 && 63.0 && 71.4 & 44.8 & 26.5\\
DeepWalk\citep{perozzi2014deepwalk}                && 43.2 && 67.2 && 65.3 && 79.5 & 72.5 & 58.1\\
ICA\citep{lu2003link}                              && 69.1 && 75.1 && 73.9 && \ \ \  -- & \ \ \  -- & \ \ \  --\\
Planetoid (Transductive)\citep{yang2016revisiting} && 64.9 && 75.7 && 75.7 && 84.5 & \textbf{75.7} & 61.9\\
Planetoid (Inductive)\citep{yang2016revisiting}    && 64.7 && 61.2 && 77.2 && 70.2 & 59.8 & 45.4\\
GCN\citep{kipf2016semi}                            && \textbf{70.3} && 81.5 && 79.0 && 83.0$^{\dagger}$ & 67.0$^{\dagger}$ & 54.2$^{\dagger}$\\ 
GGNN$^{\ast}$\citep{li2015gated}                   && 68.1 && 77.9 && 77.2 && \textbf{84.6} & 66.2 & 59.1\\
\textbf{\ourmodelshort} (Ours)                     && 69.7 && \textbf{81.8} && \textbf{79.3} && 84.4 & 74.7 & \textbf{63.9}\\
\bottomrule
\end{tabular}
}
\caption{Classification accuracies on citation networks and knowledge graphs.
  $^\ast$ (resp. $^{\dagger}$) indicates we ran our own (resp. the released)
  implementation.
\label{exp:citation}\label{exp:nell}
}
\end{table}

\subsection{Entity Classification}

Next, we consider experimental results of entity classification task on the NELL
dataset extracted from the knowledge graph first presented in
\cite{carlson2010toward}. 
A knowledge graph consists of a set of entities and a set of directed edges
which have labels (i.e., different types of relation).
Following \cite{yang2016revisiting}, each triplet $(e_1, r, e_2)$ of entities $e_1, e_2$ and relation $r$ in the knowledge graph is split into two tuples.
Specifically, we assign separate relation nodes $r_1$ and $r_2$ to each entity
and thus obtain $(e_1, r_1)$ and $(e_2, r_2)$.
Entity nodes are associated with sparse feature vectors.
We follow \cite{kipf2016semi} to extend the number of features by assigning a
unique one-hot representation for every relation node.
This results in a $61278$-dim sparse feature vector per node.
An additional validation set of $500$ labeled nodes under the label rate $0.1\%$
as in \cite{kipf2016semi} is used for tuning hyperparameters.
The chosen hyperparameters are then used for other label rates.
The semi-supervised task here considers three different label rates $10\%$,
$1\%$, $0.1\%$ per class in the training set.
We run the released code of GCN with the reported hyperparameters in
\cite{kipf2016semi}.
Since we did not observe overfitting on this dataset, we choose the
\emph{embedding-input} variant as the input model.
The results are shown in \rTab{exp:nell}, where we see that our model outperforms competitors under the most
challenging label rate $0.001$ and obtain comparable results with
the state of the art on other label rates.

\subsection{Distantly-Supervised Entity Extraction}

Finally, we consider the DIEL (Distant Information Extraction using
coordinate-term Lists) dataset~\citep{bing2015improving}.
This dataset constructs a bipartite graph where nodes are medical entities and
texts (referred as mentions and coordinate lists in the original paper).
Texts contain some facts about the medical entities. Edges of the graph are links between entities and texts.
Each entity is associated with a pre-extracted sparse feature vector.
The goal is to extract medical entities from text given sparse feature vectors
and the graph.
As shown in \rTab{exp:datasets}, this dataset is very challenging due to its extremely large scale and very high-dimensional sparse features.
Note that we attempted to run the released code GCN model on this dataset, but
ran out of memory.
Thus, we adapted the public implementation of GCN to make it successfully run on
this dataset, and also implemented GCN with our partition-based schedule.

We follow the exact experimental setup as in
\cite{bing2015improving,yang2016revisiting}, including $10$ different data
splits, preprocessing of entity mentions and coordinate lists, and evaluation.
We randomly sample $1/5$ of the training nodes as the validation set.
We regard the top-$k$ entities returned by a model as positive instances and
compute recall$@k$ as the evaluation metric where $k=240000$ as in
\cite{bing2015improving, yang2016revisiting}.
Average recall over $10$ runs is reported in \rTab{exp:diel}, and we see that
{\ourmodelshort} outperforms all other models.
Note that since Freebase is used as ground truth and some entities are not
present in texts, the upper bound of recall given by \cite{bing2015improving} is
$61.7\%$.

\begin{table}
\centering
\begin{tabular}{@{}l|c@{}}
\hline
\toprule
Method & Recall$@$k  \\
\midrule
LP~~~\citep{zhu2003semi} &  16.20 \\
DeepWalk~~~\citep{perozzi2014deepwalk} &  25.80 \\
Feat~~~\citep{yang2016revisiting} & 	34.90 \\
DIEL~~~\citep{bing2015improving} &  40.50 \\
ManiReg~~~\citep{belkin2006manifold} &	47.70 \\
SemiEmb~~~\citep{weston2012deep} &	48.60 \\
Planetoid (Transductive)~~~\citep{yang2016revisiting} & 50.00 \\
Planetoid (Inductive)~~~\citep{yang2016revisiting} & 50.10 \\
GCN$^{\ast}$~~~\citep{kipf2016semi} & 48.14 \\ 
GCN + Partition$^{\ast}$ & 48.47 \\
GGNN$^{\ast}$~~~\citep{li2015gated} &	51.15 \\ 
\ourmodelshort & \textbf{52.11} \\
\bottomrule
\end{tabular}
\caption{Average recall on the DIEL dataset.
  $^\ast$ indicates that we ran our own implementation.}
\label{exp:diel}
\end{table}

\subsection{Comparison of Different Partition Methods}

We now compare the two partition methods we considered for our model: spectral
partition and our modified multi-seed flood fill. We use the NELL data set to
benchmark and report the average validation accuracy over $10$ runs in
\rTab{exp:partition}, in which we also report the average runtime of the
partition process.
The accuracies of the trained models do not allow for a clear conclusion as to
which method to use, and in our further experiments they seem to highly depend on
the number of subgraphs, the connectivity of input graphs, optimization and
other factors.
However, our multi-seed flood fill partition method is substantially faster and
is efficiently applicable to very large graphs.

\begin{table}
\centering
\resizebox{\columnwidth}{!}{%
\begin{tabular}{@{}lrr@{\qquad}r@{}}
\hline
\toprule
Number of subgraphs & \multicolumn{2}{c}{Spectral Partition} & Modified Multi-seed Flood Fill\\
\midrule
\phantom{0}5  & 54.8\% &    (2.5s) & 62.0\%\quad (0.36s) \\
10            & 55.6\% &    (4.2s) & 63.1\%\quad (0.36s) \\
20            & 58.0\% &   (12.2s) & 57.5\%\quad (0.43s) \\
30            & 60.1\% & (3115.0s) & 59.9\%\quad (0.23s) \\
\bottomrule
\end{tabular}
}
\caption{Accuracy and run time of different partition methods with different numbers of subgraphs.}
\label{exp:partition}
\end{table}

\subsection{Comparison of Different Propagation Schedules}\label{exp:prop_schedule}

Besides the synchronous and our partition based propagation schedules, we also
investigated two further schedules based on a sequential order and a series of
minimum spanning trees (MST).

To generate a sequential schedule, we first perform graph traversal via breadth
first search (BFS) which gives us a visiting order.
We then split the edges into those that follow the visiting order and those that
violate it.
The edges in each class construct a directed acyclic graph (DAG), and we
construct a propagation schedule from each DAG following the principle that
every node will send messages once it receives all messages from its parents and
updates its own state.
An example of the schedule is given in the appendix.
Note that this sequential schedule reduces to a standard bidirectional recurrent
neural network on a chain graph.

For the MST schedule, we find a sequence of minimum spanning trees as follows. We first assign random positive weights between $0$ and $1$ to every edge and then apply Kruskal's algorithm to find an MST. Next we increase the weights by $1$ for edges which are present in the MST we found so far. This process is iterated until we find $K$ MSTs where $K$ is the total number of propagation steps.

We compare all four schedules by varying the number of propagation steps on the
Cora dataset. The validation accuracies are shown in Fig. \ref{fig:visualization} (c).
To clarify, assuming graph is singly connected, then the number of edges per propagation step of MST, Sequential, Synchronous and Partition in Fig. \ref{fig:visualization} (c) are $|V|-1$, $|E|$, $|E|$ and $|E|$ respectively. 
Here, $V$ and $E$ are the set of nodes and edges.
We also show the average results of $10$ runs with different random seeds on Cora in \rTab{exp:schedule_cora}. 

\begin{table}
\centering
\resizebox{\columnwidth}{!}{%
\begin{tabular}{@{}l|ccc@{}}
\hline
\toprule
Prop Step & 1 & 3 & 5 \\  
\midrule
MST & 59.94\% $\pm$ 0.89  & 71.83\% $\pm$ 0.96 & 77.1\% $\pm$ 0.72   \\
Sequential & 73.04\% $\pm$ 1.93  & 77.55\% $\pm$ 0.65 & 74.89\% $\pm$ 1.26 \\
Synchronous & 67.36\% $\pm$ 1.44 & 80.15\% $\pm$ 0.80 & 80.06\% $\pm$ 0.98  \\
Partition & 68.1\% $\pm$ 1.98 & 80.27\% $\pm$ 0.78 & 80.12\% $\pm$ 0.93 \\
\bottomrule
\end{tabular}
}
\caption{Accuracy of different partition methods with different propagation
  steps on the Cora dataset.}
\label{exp:schedule_cora}
\end{table}

In these results, the meaning of one propagation step varies.
For the synchronous schedule, a propagation step means that every node sent and
received messages once and updated its state.
For the sequential schedule, it means that messages from all roots of the two
DAGs were sent to all the leaves.
For the MST-based schedule, it means sending messages from the root to all
leaves on one minimum spanning tree.
For our partition schedules, it means one outer loop of the algorithm.
In this sense, messages are propagated furthest through the graph for the
sequential schedule within one propagate step.
This becomes visible in the results on a single propagation step, in which the
sequential schedule yields the highest accuracy.
However, when increasing the number of propagation steps, the computation graph
associated with the sequential schedule becomes extremely deep, making the
learning problem very hard.
Our proposed partition schedule performs similarly to the synchronous schedule
(while requiring less computation), and better than other asynchronous schedules
when using more than a single propagation step.

\begin{figure*}[t]
  \centering
  \begin{minipage}[b]{0.32\textwidth}
    \centering
    \includegraphics[width=\textwidth]{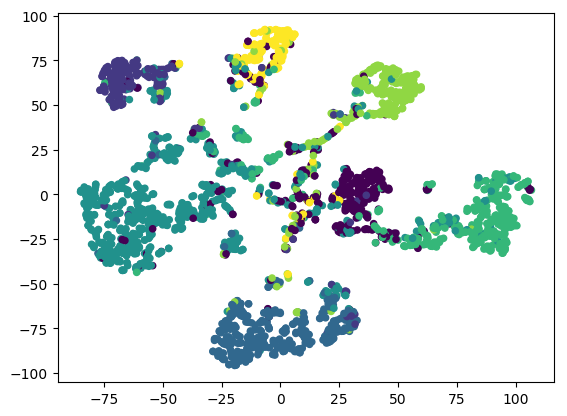} \\
    (a)
  \end{minipage}  
  \begin{minipage}[b]{0.32\textwidth}
    \centering
    \includegraphics[width=\textwidth]{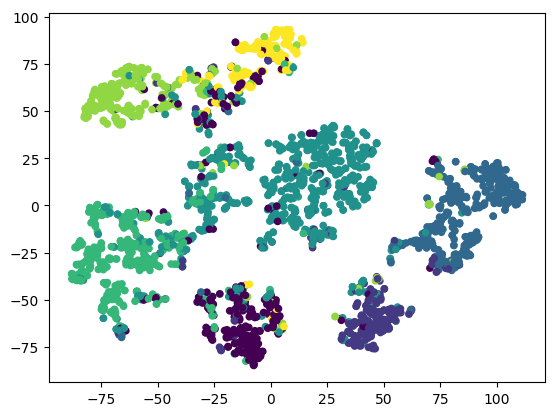} \\
    (b)
  \end{minipage}
  \begin{minipage}[b]{0.32\textwidth}
    \centering
    \includegraphics[width=\textwidth]{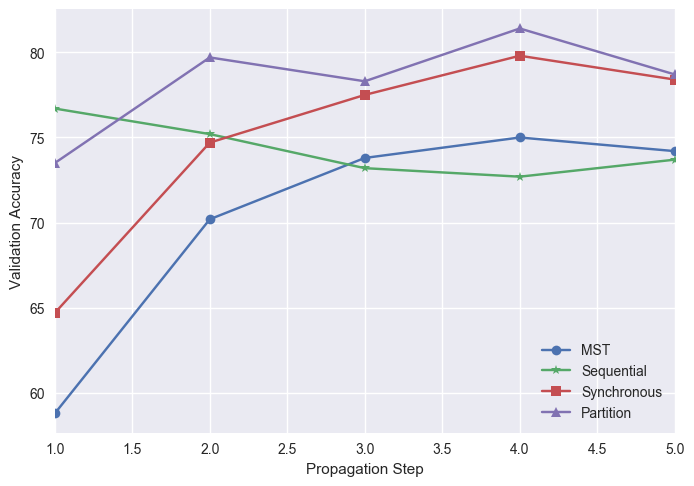} \\
    (c)
  \end{minipage}  
  \caption{(a), (b) The t-SNE visualization of node representations produced by propagation model of GGNN and {\ourmodelshort} on Cora dataset in which nodes actually belong to $7$ classes. (c) Comparison of different propagation schedules with varying propagation steps.} 
  \label{fig:visualization}
\end{figure*}


%% file: discussion.tex
\section{Conclusion}

We presented graph partition neural networks, which extend graph neural
networks.
Relying on graph partitions, our model alternates between locally propagating
information between nodes in small subgraphs and globally propagating
information between the subgraphs.
Moreover, we propose a modified multi-seed flood fill for fast partitioning of
large scale graphs.
Empirical results show that our model performs better or is comparable to
state-of-the-art methods on a wide variety of semi-supervised node
classification tasks.
However, in contrast to existing models, our GPNNs are able to handle extremely
large graphs well.

There are quite a few exciting directions to explore in the future. One is to
learn the graph partitioning as well as the GNN weights, using a soft
partition assignment. Other types of propagation schedules which have proven useful in probabilistic graphical models are also worthwhile to explore in the context of GNNs. To further improve the efficiency of propagating information, different nodes within the graph could share some memory, which mimics the shared memory model in the theory of distributed computing. Perhaps most importantly, this work makes it possible to run GNN models on very large graphs, which potentially opens the door to many new applications.


%% file: appendix.tex
\section{Appendix}

\subsection{Bi-directional Chain}
\label{sec:chain}

In this section, we revisit the broadcast problem on bi-direction chain graphs. We show that our propagation schedule has advantages over the synchronous one via the following proposition.

\begin{prop}
Let $\graph$ be a bi-direction chain of size $N$. We have:
(1) Synchronous propagation schedule requires $2(N-1)^2$ messages to solve the problem;
(2) If we partition the chain evenly into $K$ sub-chains for $1 \le K \le N$, GPNN propagation schedule can solve the problem with $2((N-K)^2+(K-1)^2)$ messages.
\end{prop}

\begin{proof}
We first analyze the case for synchronous propagation schedule. At each round, it needs $2(N-1)$ messages to propagate messages one step away. Since it requires at least $(N-1)$ steps for message from one endpoint of the chain to reach the other, the number of messages to solve broadcast is thus $2(N-1)^2$.

We now turn to our schedule. Since the chain is evenly partitioned, each sub-chain is of $n=N/K$ nodes. We need to perform $(n-1)$ propagation steps to traverse a sub-chain, so we set $T_S=n-1$. The number of messages required by a single sub-chain during the intra-subgraph propagation phase is $2(n-1)^2$, and so all $K$ sub-chains collectively require $2K(n-1)^2$ messages. Between intra-subgraph propagation, we perform $T_C=1$ step of inter-subgraph propagation to transfer messages over the cut edges between sub-chains. Each inter-subgraph step requires $2$ messages per cut edge - i.e. 2(K-1) messages in total. We need $K$ outer loops to ensure that message from any node can reach any other nodes, and strictly speaking, the the last inter-subgraph propagation step is unnecessary. So in total, we require $K\times 2K(n-1)+(K-1)\times 2(K-1) = 2((N-K)^2+(K-1)^2)$ messages, which proves the proposition.
\end{proof}

One can see from the above proposition that if we take $K = 1$ and $K = N$, the number of messages of our schedule matches the synchronous one. We can also derive the optimal value of $K$ as $(N+1)/2$ resulting in a factor of $2$ reduction in the total messages sent compared to the synchronous schedule.

\subsection{Hyperparameters}

We train all models using Adam \cite{kingma2014adam} with a learning rate of $0.01$. 
We also use early stopping with a window size of $10$. We clip the norm gradient to ensure that it is no larger than $5.0$. 
The maximum epoch of all experiments except NELL is set to $100$. The one of NELL is $300$. 
The weight decays for Cora, Citeseer, Pubmed, NELL and DIEL are set to $7.0e^{-4}$, $5.0e^{-4}$, $9.0e^{-4}$, $7.0e^{-4}$ and $1.0e^{-5}$ respectively. 
The dimensions of state vectors of \ourmodelshort for Cora, Citeseer, Pubmed, NELL and DIEL are set to $128$, $128$, $128$, $512$ and $64$. The output model for Cora, Citeseer, NELL is just softmax layer. For Pubmed and DIEL, we add one hidden layer with $tanh$ activation function before the softmax which have dimension $512$ and $2048$ respectively.

\subsection{Random Splits of Citation Networks}

We include the results on citation networks with $10$ random splits in Table \ref{exp:citation_rand}. 
From the table, we can see that our results are comparable with the state-of-the-art on these small scale datasets.

\begin{table}
\centering
\begin{tabular}{@{}l|ccc@{}}
\hline
\toprule
Method & Citeseer & Cora & Pubmed  \\
\midrule
GCN$^{\dagger}$~~~\citep{kipf2016semi} & \textbf{68.7} $\pm$ 2.0 & \textbf{80.4} $\pm$ 2.8 & \textbf{77.5} $\pm$ 2.1 \\
GGNN$^{\ast}$~~~\citep{li2015gated} & 66.3 $\pm$ 2.0 & 78.9 $\pm$ 2.6 & 74.7 $\pm$ 2.8 \\
\ourmodelshort & 68.6 $\pm$ 1.7 & 79.9 $\pm$ 2.4 & 76.1 $\pm$ 2.0 \\
\bottomrule
\end{tabular}
\caption{Classification accuracies on citation networks with $10$ random splits. $^\ast$ and $\dagger$ indicates we run our own implementation and the released code respectively.}
\label{exp:citation_rand}
\end{table}

\subsection{Sequential Propagation Schedule}

In Fig. \ref{fig:vis_seq} we show an example visualization of the DAGs decomposition of the sequential propagation schedule we implemented in the section \ref{exp:prop_schedule}. 

\begin{figure*}[t]
  \centering
  \begin{minipage}[b]{0.32\textwidth}
    \centering
    \includegraphics[width=\textwidth]{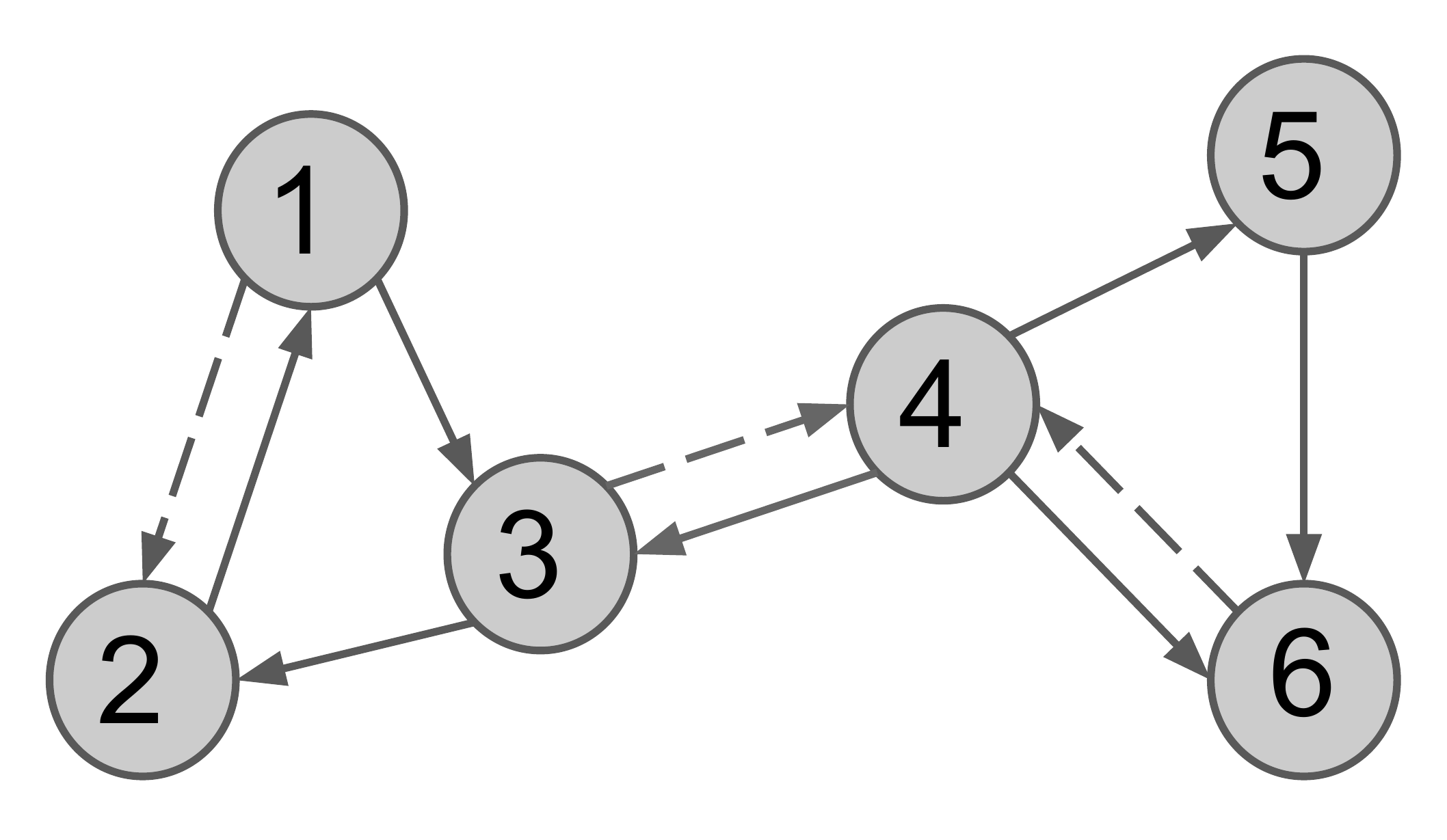} \\
    (a)
  \end{minipage}  
  \begin{minipage}[b]{0.32\textwidth}
    \centering
    \includegraphics[width=\textwidth]{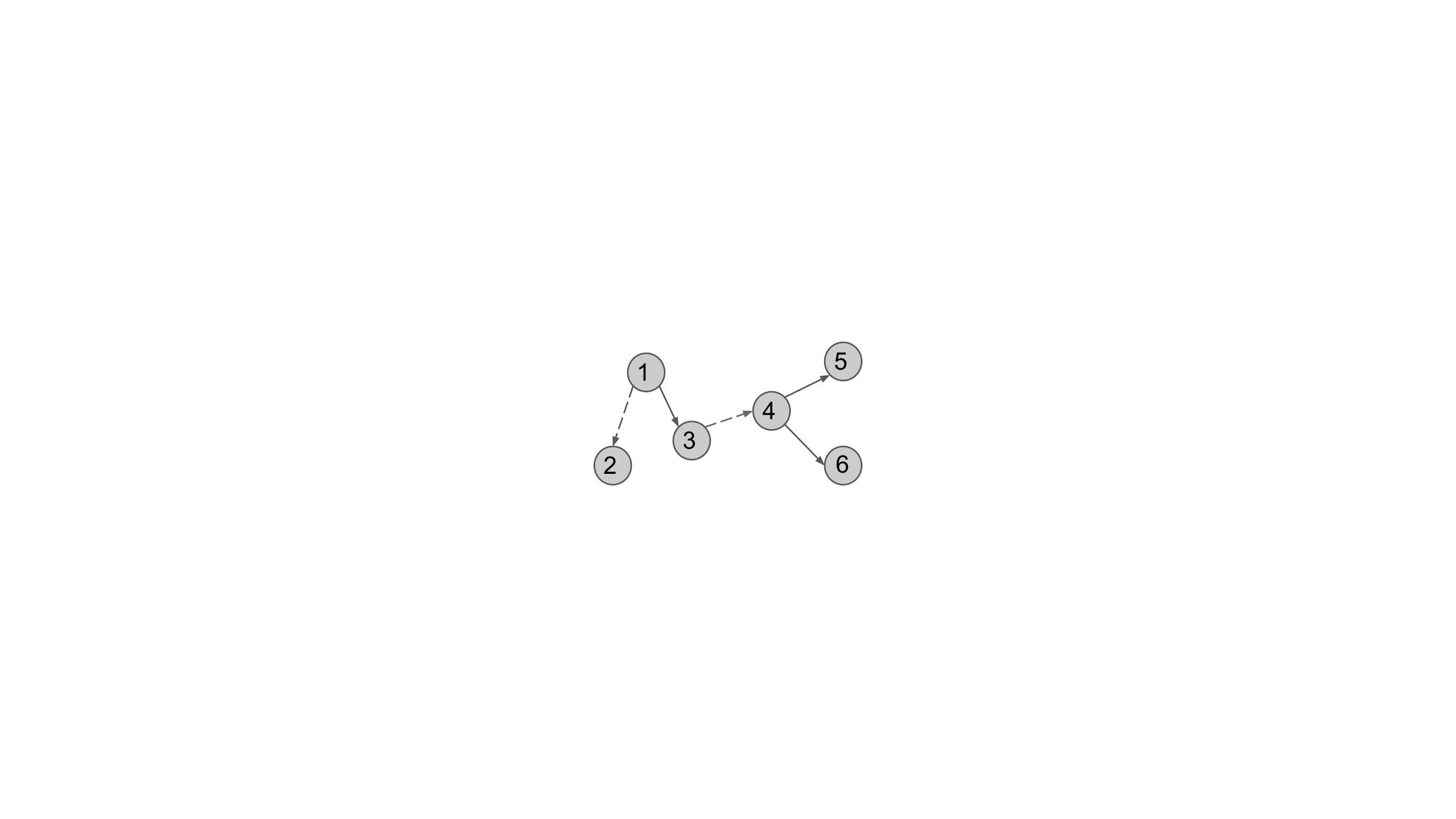} \\
    (b)
  \end{minipage}
  \begin{minipage}[b]{0.32\textwidth}
    \centering
    \includegraphics[width=\textwidth]{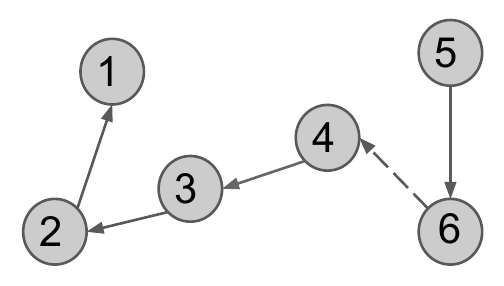} \\
    (c)
  \end{minipage}  
  \caption{Sequential scheduling. (a) The original graph. (b) and (c) are the two DAGs obtained by the sequential schedule we described in section \ref{exp:prop_schedule} where BFS traversal is started from node $1$.} 
  \label{fig:vis_seq}
\end{figure*}

\subsection{Random Partition Schedule}

We did an experiment on schedules which are determined by random partitions of the graph. 
In particular, for $k$-step propagation, we randomly sample $1/k$ proportion of edges from the whole edge set without replacement and use them for update. 
We summarize the results ($10$ runs) on the Cora dataset in Table \ref{exp:rand_schedule}.
\begin{table}
\centering
\begin{tabular}{@{}l|cccc@{}}
\hline
\toprule
Propagation Step & 2 & 3 & 5 & 10 \\
\midrule
Avg Acc & 76.03 & 74.71 & 72.09 & 69.99 \\
Std Acc & 1.55 & 1.31 & 1.81 & 2.26 \\ 
\bottomrule
\end{tabular}
\caption{Classification accuracies on Cora network with random partition based schedule.}
\label{exp:rand_schedule}
\end{table}

From the results, we can see that the best average accuracy $(K = 2)$ is $76.03$ which is still lower than both synchronous and our partition based schedule. 
Note that this result roughly matches the one with spanning trees. 
The reason might be that random schedules typically need more propagation steps to spread information throughout the graph. 
However, more propagation steps of GNNs may lead to issues in learning with BPTT.

\subsection{Implementation}

The released code of GGNN~\citep{li2015gated} is implemented in Torch. We implement both our own version of GGNN and our model in Tensorflow~\citep{tensorflow2015-whitepaper}. To ensure correctness, we first reproduced the experimental results of the paper on bAbI artificial intelligence (AI) tasks with our implementations of GGNN. Our code will be released soon.
One challenging part is the implementation of synchronous propagation within
subgraphs.
We implicitly implement the parallel part by building one separate branch of the
computational graph for each subgraphs (i.e., use a Python \texttt{for} loop
rather than \texttt{tf.while\_loop}).
This relies on the claim that tensorflow optimizes the execution of
the computational graph in a way that independent branches of the graph will be
executed in parallel as decribed in~\cite{tensorflow2015-whitepaper}.
However, since we have no control of the optimization of the computational graph,
this part could be improved by explicitly putting each branch on one separate
computation device, just like the multi-tower solution for training convolutional
neural networks (CNNs) on multiple GPUs.